\documentclass{article}
\usepackage{times}
\usepackage{enumerate}

\usepackage{color}
\usepackage{xcolor}
\definecolor{mydarkblue}{rgb}{0,0.08,0.45}
\usepackage[hidelinks,colorlinks]{hyperref}
\hypersetup{
	colorlinks=true,
	linkcolor=mydarkblue,
	citecolor=mydarkblue,
	filecolor=mydarkblue,
	urlcolor=mydarkblue,
}

\usepackage{graphicx}
\usepackage{natbib}
\usepackage{mathrsfs,amssymb,amsmath,textcomp,amsfonts,amsthm,bbm,fullpage}

\usepackage{color}
\usepackage{multirow,amsmath, dsfont, amscd, latexsym,color,amssymb,setspace}
\usepackage{ifthen}
\usepackage{enumerate}

\usepackage{listings}
\usepackage{enumitem}
\usepackage{etoolbox}
\usepackage{caption}
\usepackage{thmtools,thm-restate}
\usepackage{comment}
\usepackage{subcaption}
\usepackage{url}
\usepackage{graphicx,caption, subcaption, multirow}
\usepackage{float}

\newtheorem{prop}     {Proposition}

\newcommand{\ip}[2]{
{\langle {#1} , {#2} \rangle}
}
\newcommand{\vdot}[2]{
	{ {#1}^T {#2}}
}

\providetoggle{solution}
\settoggle{solution}{false}

\newcommand{\cT}{\mathcal{T}}

\newcommand*{\Prob}{\mathbb{P}}

\newcommand{\norm}[1]{\| #1 \|_2}
\newcommand{\inftynorm}[1]{\| #1 \|_{\infty}}

\newtheorem{lem}{Lemma}

\newtheorem*{Conj*}{Conjecture}

\newcommand{\cZ}{{\cal Z}}

\newcommand{\td}[1]{\tilde{#1}}
\newcommand{\eat}[1]{}

\newcommand{\commented}{no}

\ifthenelse{\equal{\commented}{yes}}{
\newcommand{\pnote}[1]{\footnote{{\bf [[Parik: {#1}\bf ]] }}}
}

\ifthenelse{\equal{\commented}{no}}{
\newcommand{\pnote}[1]{}
}

\newcommand{\ignore}[1]{}

\def\colorful{3}

\ifnum\colorful=1

\fi

\newcommand{\E}{\mathop{\mathbb{E}\/}}

\newcommand{\R}{\mathbb R}

\newcommand{\eps}{\epsilon}

\newcommand{\poly}{\mathrm{poly}}

\newcounter{this-list}

\newcommand{\significant}{\Prob_{x|s}}
\newcommand{\posterior}{\Prob_{x|v}}
\newcommand{\initial}{\Prob_{x|v_0}}
\newcommand{\edgeposterior}{\Prob_{x|e}}
\newcommand{\Sig}{\text{Sig}}
\newcommand{\Bad}{\text{Bad}}
\newcommand{\cS}{{\mathcal{S}_d}}
\newcommand{\cU}{\mathcal{U}_d}
\newcommand{\edge}{e}

\newtheorem*{conjecture}{Conjecture}

\begin{document}


\vspace{.2in}


\newcommand{\theTitle}{Memory-Sample Tradeoffs for Linear Regression with Small Error\thanks{The contributions of Vatsal and Gregory were supported by NSF awards AF-1813049 and CCF-1704417, and ONR Young Investigator Award N00014-18-1-2295. Aaron's contributions were supported by NSF CAREER Award CCF-1844855.}}
\author{
 \fontsize{11}{13}\selectfont {\bf Vatsal Sharan} \\
 \fontsize{11}{13}\selectfont Stanford University \\
 \fontsize{11}{13}\selectfont {\tt vsharan@stanford.edu}
\and
 \fontsize{11}{13}\selectfont {\bf Aaron Sidford} \\
 \fontsize{11}{13}\selectfont Stanford University \\
 \fontsize{11}{13}\selectfont {\tt sidford@stanford.edu}
\and
\fontsize{11}{13}\selectfont {\bf Gregory Valiant} \\
\fontsize{11}{13}\selectfont Stanford University \\
\fontsize{11}{13}\selectfont {\tt valiant@stanford.edu}
}

\title{\theTitle}
\date{}

\clearpage
\maketitle\begin{abstract}
	We consider the problem of performing linear regression over a stream of $d$-dimensional examples, and show that any algorithm that uses a subquadratic amount of memory exhibits a slower rate of convergence than can be achieved without memory constraints.  Specifically, consider a sequence of labeled examples $(a_1,b_1), (a_2,b_2)\ldots,$ with $a_i$ drawn independently from a $d$-dimensional isotropic Gaussian, and where $b_i = \langle a_i, x\rangle + \eta_i,$ for a fixed $x \in \mathbb{R}^d$ with $\|x\|_2 = 1$ and with independent noise $\eta_i$ drawn uniformly from the interval $[-2^{-d/5},2^{-d/5}]$.  We show that any algorithm with at most $d^2/4$ bits of memory requires at least $\Omega(d \log \log \frac{1}{\epsilon})$ samples to approximate $x$ to $\ell_2$ error $\epsilon$ with probability of success at least $2/3$, for $\epsilon$ sufficiently small as a function of $d$.  In contrast, for such $\epsilon$,  $x$ can be recovered to error $\epsilon$ with probability $1-o(1)$ with memory $O\left(d^2 \log(1/\epsilon)\right)$ using $d$ examples.  This represents the first nontrivial lower bounds for regression with super-linear memory, and may open the door for strong memory/sample tradeoffs for continuous optimization.
\end{abstract}
\thispagestyle{empty}
\newpage
\setcounter{page}{1}
\section{Introduction}
What are the implications of memory constraints on the ability to efficiently learn or optimize?  As has been revealed in a recent series of striking results~\citep{raz2016fast,raz2017time,beame2018time,kol2017time,moshkovitz2017mixing,moshkovitz2017mixings,garg2018extractor}, for a broad class of natural learning problems over the Boolean hypercube and other finite fields, there is a sharp threshold for the amount of memory required to learn with a polynomial amount of data.

This line of work was sparked by Raz' breakthrough result~\citep{raz2016fast}, which considered the problem of learning a parity: given access to a stream of labeled examples, $(a_1,b_1),\ldots,$ where each $a_i \in \{0,1\}^d$ is drawn uniformly from the $d$-dimensional hypercube and $b_i = \langle a_i,x\rangle \mod 2$ for some fixed vector $x$, Raz showed that any algorithm with $o(n^2)$ memory would require an exponential number of examples to learn $x$ (with any significant probability of success).  Of course, given a quadratic amount of memory, $x$ can be efficiently computed by taking the first $O(n)$ examples and solving the corresponding linear system over $\mathbb{F}_2$.  Subsequent work extended this result to a broad class of discrete learning problems, including~\cite{raz2017time,moshkovitz2017mixing,moshkovitz2017mixings} which generalized the results to a class of Boolean learning problems that satisfy a certain combinatorial condition,~\cite{kol2017time} which extended the techniques to the problems of learning sparse parities (parities involving $o(d)$ coordinates) which implies hardness of several other natural Boolean learning problems including learning small juntas, small decision trees, and small DNF formulae, and the works~\cite{beame2018time} and~\cite{garg2018extractor} whcih strengthened the approach of~\cite{raz2017time} to yield tight tradeoffs for a larger class of learning problems over finite fields, including homogeneous $m$-variate polynomials over $\mathbb{F}_2$.

For continuous, real-valued optimization and learning problems, much less is known about memory/sample tradeoffs.  This is in spite of fact that the problem of learning a linear regression---the real-valued analog of learning parities---lies at the core of machine learning and is a prototypical convex optimization problem.  Indeed, one of the original motivations for the conjecture that learning a parity required either a quadratic memory or exponential time, originally stated in~\cite{steinhardt2016memory}, was the question of the memory/sample tradeoffs for linear regression. 

This question of the memory/sample tradeoffs for linear regression is also extremely important from a practical perspective. Gradient-based `first-order' methods are \emph{the} workhorse of modern machine learning, in contrast to `second-order' methods.   This is explained by the efficiency benefit conferred by the linear memory footprint of first-order methods as opposed to the quadratic memory requirements of second-order methods.  For large-scale learning problems, this reduction in memory usage of first-order methods more than compensates for the increase in the number of iterations or datapoints processed.  If methods with comparable memory usage to first-order methods (or at least significantly subquadratic) were capable of achieving similar convergence rates to second-order methods, that could have far-reaching practical implications.  

The question of the memory/sample tradeoffs for linear regression is also a natural and largely unexplored frontier of continuous optimization research. There is classic line of research which has proven information theoretic lower bounds on continuous optimization \citep{nemirovskii1983problem,Nesterov2004notes,Bubeck15} in a restricted oracle model where the input real-valued function can only be queried via black box queries to an oracle that returns local information about the function, e.g. function values, gradients, etc. Given only mild regularity assumptions on a function, e.g. Lipschitz continuity, smoothness, convexity, etc., proving tight bounds on the number of queries to an oracle needed to approximately minimize a function is well-studied and has been a driving force behind the development of modern optimization theory. While there is some work studying the effect of parallelism on these lower bounds \citep{Nemirovski94, Yaron2018, DiakonikolasGuzman18}, we are unaware of previous work proving gaps between the query complexity of optimization problems under differing memory constraints. 
\medskip

In this work, we provide the first nontrivial memory/sample tradeoffs for linear regression which apply in the regime where the available memory is significantly larger than what would be required to store each datapoint to high precision.

\medskip
\noindent {\textbf{Theorem~\ref{thm:main}.}} \emph{Consider a sequence of labeled examples $(a_1,b_1), (a_2,b_2)\ldots,$ with $a_i$ drawn independently from a $d$-dimensional isotropic Gaussian with an identity covariance matrix, and where $b_i = \langle a_i, x\rangle + \eta_i,$ for a fixed $x \in \mathbb{R}^d$ with $\|x\|_2 = 1$ and with independent noise $\eta_i$ drawn uniformly from the interval $[-2^{-d/5},2^{-d/5}]$. Let $\eps= 1/d^r$ for any $r\le O(d/\log d)$. Then any algorithm with at most $d^2/4$ bits of memory requires at least $\Omega(d \log r)$ samples to approximate $x$ to $\ell_2$ error $\eps$ with probability of success at least $2/3$. In particular, for $\eps\le 1/d^{\Omega(\log d)}$, this implies that any algorithm with $d^2/4$ bits of memory requires at least $\Omega(d \log \log \frac{1}{\eps})$ samples to approximate $x$ to error $\eps$ with probability of success at least $2/3$. }
\medskip

For comparison, note that for  $\eps\in[2^{-\Theta(d)}, 1/d]$, the trivial algorithm based on storing $d$ examples in memory and then solving the linear system recovers $x$ to error $\eps$ with probability $1-o(1)$   using $O\left(d^2 \log(1/\eps)\right)$ bits of memory and $d$ examples.\footnote{This follows from the fact that the condition number of the system with $d$ examples is at most $1/\text{poly}(d)$ with high probability, and hence we can solve the system to accuracy $\eps$ by doing all computations with $O(\log(d/\eps))=O(\log(1/\eps))$ bits of precision, for $\eps\le 1/d$.}


\subsection{Further Directions}

Our result establishes the existence of a sharp gap in sample complexities for regression with bounded memory. Nevertheless, this still leaves a significant margin between our lower bounds on the convergence rate of bounded memory algorithms, and those achieved by the best known first-order methods which use memory $O(d \log(1/\epsilon))$.
For example, randomized Kaczmarz \citep{StrohmerV06} can be easily shown to compute a point $\td{x}$ such that $\norm{x - \td{x}} \leq \epsilon$ with constant probability using $O(d \log(1/\epsilon))$ samples and memory $O(d \log(1/\epsilon))$. This is the best known sample complexity for achieving error $\eps$ given this amount of memory, though our lower bound leaves open the question of whether it is optimal.

Beyond tightening our result, it is also worth considering the analogous regression question in the setting where the datapoints are drawn from an ill-conditioned distribution.  Both in practice, and theory, first-order methods suffer a convergence rate that degrades with the condition number. Over the past decade there has been extensive research on providing iterative methods which use sub-quadratic space and seek better dependencies on the eigenvalues of the covariance of the distribution of $a$. Though there have been several improvements to randomized Kaczmarz and variants of SGD in recent years \citep{HuKP09, Lan12, GhadimiL13, LeeS13, FrostigGKS15, lw2016, NeedellSW16, Dieuleveut2017, jkkns2018} the number of samples required by these methods all depend \emph{polynomially} on some measure of eigenvalue range or conditioning of the underlying covariance matrix. This is in sharp contrast to  second-order methods which can simply store $\Theta(d)$ samples and invert an associated matrix to compute an $\epsilon$ accurate solution with $\Omega(d^2 \log(d \kappa / \epsilon))$ memory where $\kappa$ is a condition number measure of the matrix. 

\begin{conjecture}
For any $\kappa$ bounded by a polynomial of $d$, there exists a distribution $\mathcal{D}_{\kappa}$ over $d$-dimensional Gaussian distributions whose covariance has condition number $\kappa$, such that for $G \leftarrow \mathcal{D}_{\kappa}$, given a sequence of examples $(a_i,b_i)$ with $a_i$ drawn from $G$ and $b_i = \langle a_i, x\rangle$, any algorithm that recovers the unit vector $x$ to small constant $\ell_2$ error with constant probability either requires  $\Omega(d^2)$ bits of memory, or $d \cdot \poly(\kappa)$ examples.
\end{conjecture}

The work of this paper may be a first step towards proving the above conjecture. We hope this paper will inspire efforts to establish strong memory/query complexity trade-offs for continuous optimization more broadly. 

In a different direction, it may be worth considering the extent to which results of the form of Theorem~\ref{thm:main} apply beyond the stochastic streaming setting.  Rather than considering a stream of independent examples, one could consider the analogous questions in a cell-probe setting: suppose there is a set of $O(d)$ examples stored in read-only memory, and one is charged according to the number of times each example is `downloaded' into working memory.  What are the tradeoffs between the amount of working memory, error of the recovered linear regression, and number of `downloads'?   This setting closely corresponds to the data pipeline employed in many large scale learning settings, and any strong results in this setting would be extremely interesting.

It is worth noting that, even in the setting of learning parities, the stochasticity of the examples is essential to the exponential sample complexity of memory-constrained algorithms.  Analogous results are not true in the above cell-probe model.  For example, given $O(n)$ examples stored in read-only memory, there exists a successful learning algorithm for the parity problem with $O(n)$ working memory that uses $poly(n)$ runtime and cell-probes \citep{weihao} (though, to the best of our knowledge, it is not known if there is a successful learning algorithm for the real-valued regression problem which uses $O(n)$ working memory and $poly(n)$ cell-probes).  Still, establishing any nontrivial gap between memory-constrained and unconstrained learning (for either the real-valued regression or parity problems) in the cell-probe setting would be exciting, though may be quite difficult.

\subsection{Related Work}

A number of recent works have examined learning problems such as sparse linear regression \citep{steinhardt2015minimax} and detecting correlations \citep{shamir2014fundamental,dagan2018detecting} under information constraints such as limited memory or communication constraints. These results usually develop information-theoretic inequalities \citep{braverman2016communication,steinhardt2016memory,dagan2018detecting} to show that unless a set of parties exchange a minimum amount of information, they cannot solve the learning problem---with the memory bound following as a consequence of the communication lower bound. At a high level, the idea is to show that if the learning problem requires distinguishing between a set of $k$ distributions, and if the distributions are sufficiently uncorrelated, then at least $\Omega(k)$ bits of communication are needed to solve the learning problem. While initial results only obtained lower bounds for settings where the memory budget is less than the size of each data point, the recent work \citet{dagan2018detecting} circumvented this barrier and showed strong lower bounds for detecting correlations for natural distributions under information constraints.

Many of these information theoretic tools seem to break down for learning problems such as parity learning where  communication lower bounds do not directly give meaningful memory bounds. Hence these settings require explicitly taking into account the memory constraint of the algorithm; the recent line of work discussed in the introduction, starting with \citet{raz2016fast} achieves sharp lower bounds for memory-bounded learning by directly analyzing the structure of width-bounded branching programs for these problems \citep{raz2017time,kol2017time,moshkovitz2017mixing,moshkovitz2017mixings,beame2018time,garg2018extractor}.  Our work directly builds on the analysis framework developed in~\cite{raz2017time}, and extended in~\cite{beame2018time} and~\cite{garg2018extractor}, with the crucial difference that the geometry of the continuous space corresponding to linear regression lacks many of the combinatorial properties that are leveraged in the analysis of these prior works.

There is also a large literature on memory lower bounds for streaming algorithms (for e.g. ~\cite{alon1999space,bar2004information,clarkson2009numerical}, although these are mostly for non-learning problems and assume that the input stream is constructed in an adversarial fashion.

On the optimization side, there is a long history of proving information theoretic lower bounds on optimization methods. These results typically show that given a type of restricted local oracle to access the input, i.e. an oracle which only returns information about values, gradients, higher derivatives, separation oracles, etc., lower bounds can be formally proven on the number of queries needed to approximately minimize the function. Such results date back to the early work of \cite{nemirovskii1983problem} on the oracle complexity of optimization and there are too many results to do a complete review (see \citet{Nesterov2004notes,Bubeck15}
for more recent surveys). Key results in this broad area of research include, tight oracle bounds known for computing approximate minimizer of smooth convex functions given a gradient oracle \citep{nemirovskii1983problem, Nesterov2004notes, Bubeck15} even when randomization is allowed \citep{WoodworthS16}, tight oracle bounds known for computing approximate minimizer of Lipschitz convex function given by a subgradient oracle \citep{BraunGP17, nemirovskii1983problem, Nesterov2004notes}, and even tight oracle bounds for computing critical points, that is points of small gradient, for smooth non-convex functions given by a gradient oracle \citep{nesterov2012make,cartis2012much,cartis2017worst,carmon2017low1,carmon2017lower2}. There has also been extensive research on the oracle complexity of stochastic optimization \citep{nemirovskii1983problem,devolder2013first,devolder2014first,Shamir13,DuchiJWW15} and work on the tradeoff between oracle complexity and parallelism for nonsmooth optimization with a subgradient oracle \citep{Nemirovski94, Yaron2018, DiakonikolasGuzman18}. However, to the best of the authors knowledge the problem of memory / query complexity tradeoffs for real-valued continuous optimization has been largely unexplored.

\section{Setup and Proof Overview}

In this section, we provide an overview of our proof approach.  We begin by describing the notation and formalism we will use in analyzing the branching program representing a memory-bounded learning algorithm. 

\subsection*{Branching Programs for Learning}

We model the learner by a branching program $B$. A branching program is a general non-uniform model for space bounded computation. The branching program has $m$ layers, corresponding to $m$ time steps, with each layer having at most $w$ vertices, where $w$ denotes the \emph{width} of $B$. Each vertex of $B$ corresponds to a memory state, and a branching program with width $w$ corresponds to an algorithm with memory usage $\log_2 w$. A vertex with no outgoing edges is called a \emph{leaf}, and all vertices in the last layer are leaves (though there may be additional leaves). Each non-leaf vertex $v$ has an associated transition function $f_v: \mathbb{R}^d \times \mathbb{R} \rightarrow [w]$, representing the mapping from an example $(a,b)$ to a vertex in the next layer. Without loss of generality, we may assume that these transition functions are deterministic, as randomization cannot improve the probability of success.\footnote{This can be easily seen by noting that any branching program with randomized transitions can be  converted to a deterministic one by iteratively derandomizing each vertex by replacing its randomized transition function with the deterministic one that select the transition that maximizes the probability of success (breaking ties arbitrarily), where the probability is taken over the randomization in the subsequent examples and whatever randomization remains in the transition functions corresponding to other vertices.}  To be consistent with the literature on branching programs, we will refer to this transition function as a series of `edges' indexed by the (infinite number of) potential examples $(a,b).$ Finally, each leaf vertex, $v$, of the branching program is labeled by a \emph{label}, $\tilde{x}(v)$, representing the output value that the corresponding algorithm would produce on the sequence of examples that led to vertex $v$. 

The success probability of the branching program $B$ for a specified accuracy parameter $\eps$ is the probability that ${\norm{\td{x}-x}\le \eps}$, where $\td{x}$ is the vector returned by $B$, and the probability is with respect to the randomness in the sequence of examples and choice of the true $x$.

\subsection*{Setup}

We consider branching programs whose goal is to learn some true $x\in \mathbb{R}^d$ with $\norm{x}=1$ to $\ell_2$ error $\eps$, in the setting where $x$ is drawn uniformly at random from the $d$ dimensional unit sphere. At every time step, a $d$-dimensional vector $a$ is sampled from $N(0,I_d)$, and the branching program is given $a$ and the (noisy) inner product $b=a^Tx+\eta$, where the noise $\eta$ is sampled from $U[-\delta,\delta]$ for $\delta=2^{-d/5}$.  The addition of this noise facilitates the analysis, and we could have equivalently assumed that the true inner product $a^Tx$ is discretized according to some exponentially small discretization error $\delta$. Note that as long as the goal is to estimate $x$ up to accuracy $\eps \ge 2^{-\gamma d}$ for a small constant $\gamma$, the small uniform noise or discretization does not create any information theoretic obstacles. 

\subsection*{Proof Overview}\label{sec:proof_overview}

Our proof follows and builds on the recent analytic framework for showing time-space lower bounds developed in \citet{raz2017time}, and further extended in \citet{beame2018time} and \citet{garg2018extractor}. The analysis in our case is complicated by the fact that the gap in the sample complexity of first order and second order methods for regression on well-conditioned matrices is not very large, and depends on the desired error $\eps$. 
To capture this dependence of the sample complexity on $\eps$, we divide the branching program into multiple \emph{stages}, where a stage is a group of consecutive layers of the branching program. Each stage will intuitively correspond to the branching program reducing the $\ell_2$ error of the estimate of $x$ by a factor of two. We will argue that each of these stages cannot be too short if the algorithm has small memory. We now sketch the proof, describing the high-level framework of \citet{raz2017time} and how we adapt it to our setup.

As in \citet{raz2017time}, we define a truncated computation path $\cT$ which follows the computation path of the branching program $B$, but may stop before reaching a leaf vertex. The conditions under which $\cT$ stops before reaching a leaf vertex will depend on the stage of the branching program which $\cT$ is in. For any vertex $v$ in the branching program, let $\posterior$ be the posterior distribution of $x$ conditioned on being at $v$. We will quantify the progress made by a vertex $v$ towards learning $x$ by the $\ell_2$ norm $\norm{\posterior}$ of the posterior distribution $\posterior$ of $v$, note that a large norm indicates a concentrated posterior with an accurate estimate of $x$. The truncated path $\cT$ stops at any \emph{significant} vertex $s$ where $\norm{\significant}$ is larger than some threshold, where the threshold is chosen as a function of the  stage of the branching program being analyzed. Intuitively, if $\norm{\significant}$ is larger than a given threshold, then $\cT$ has more information about $x$ then we expect it to have at that stage. Most of the effort in \citet{raz2017time} and in our work goes into ensuring that the probability of any significant vertex is small enough  that the probability of $\cT$ stopping due to reaching a significant vertex is small.

We now sketch the argument for showing that the probability of reaching a significant vertex, $s,$ is small, for any stage $B_t$ of the branching program $B$. Let $L_i$ denote the set of all vertices in the $i$th layer of the $t$-th stage $B_t$. The following potential function tracks the progress which the $i$-th layer of $B_t$ has made towards a fixed significant vertex $s$ in $B_t$,
\begin{align*}
 \cZ_i = \sum_{v \in L_i}^{}\Pr(v)\cdot\ip{\posterior}{\significant}^{d/2} .
 \end{align*}
We claim that if  $\cZ_i$ is small and the significant vertex $s$ lies in the $i$-th layer of $B_t$, then the probability of $s$ must also be small. This follows because we define significant vertices as those for which $\norm{\significant}$ is large, and if $s$ is in the $i$th layer then $\cZ_i \ge \Pr(s)\cdot\norm{\significant}^{d}$. Hence our goal will be to show that $\cZ_i$ is small. Note that raising to the power of $d/2$ in our expression for $\cZ_i$ allows us to show that the probability of significant vertices is small enough that we can do a union bound over all vertices in the branching program, and $d/2$ is the largest power to which we can raise while keeping the contribution of the low probability events small. 

We prove that $\cZ_i$ is small via an induction argument. We first show that $\cZ_0$ must be small, as the previous stage $B_{t-1}$ of the branching program could not have made too much progress. We next show that $\cZ_{i+1}$ cannot be much larger than $\cZ_i$. To show this, we introduce another potential which tracks how much progress any edge $e$ of the branching program has made towards $s$. Let $\edgeposterior$ be the posterior distribution of $x$ conditioned on the event of traversing edge $e$ in the branching program. Let $\Gamma_i$ denote the set of all edges from the $(i-1)$-th layer to the $i$-th layer of the $t$-th stage $B_t$ of the branching program, and let $p(e)$ be the p.d.f. of the distribution over edges evaluted at edge $e$. For any $i$, we define the potential,
\begin{align*}
\cZ_i' = \int\displaylimits_{\edge \in \Gamma_{i}}p(\edge)\cdot\ip{\edgeposterior}{\significant}^{ d/2}\; de.
\end{align*}
A straightforward convexity argument shows that $\cZ_{i+1}\le \cZ'_i$. Hence the main challenge is showing that $\cZ_i'$ cannot be much larger than $\cZ_i$. This is where our analysis differs significantly from \citet{raz2017time} (this is also where \citet{beame2018time} and \citet{garg2018extractor} differ the most from \citet{raz2017time}). In these previous works which concern learning over finite fields, the learning problem is viewed as a certain matrix, and properties of this matrix are used to show that $\cZ_i'$ cannot be much larger than $\cZ_i$. It is worth noting that in these settings, it is possible to argue that the example in the next time step looks almost random to the branching program if it does not have significant knowledge about the answer, and then use this to show that the branching program cannot make too much progress when it gets an example. In our case though, first-order methods which require only linear memory \emph{can} learn $x$ up to non-negligible error with only linear sample complexity, hence the examples do not have as much randomness. Also, as we work over continuous spaces we lack the combinatorial properties that enables the analysis in the previous works to go through, and need to develop different tools.  

We now sketch our argument for showing that $\cZ_i'$ cannot be much larger than $\cZ_i$. For intuition, we first describe the argument as it would pertain to the branching program corresponding to the linear memory, first-order method for regression. At a high level, by the end of the $t$-th stage of this branching program, the algorithm has learned $x$ up to error roughly $\eps_t=1/2^t$, and the posterior $\posterior$ of a vertex $v$ in this stage roughly corresponds to a spherical Gaussian with standard deviation $\eps_t$ in every direction.  A target significant vertex, $s$, in the $t$-th stage will have posterior $\significant$ roughly corresponding to another spherical Gaussian, but with standard deviation $\eps_t/2$ in every direction.  This significant vertex represents a memory state that has learned significantly more than is expected of a vertex in this stage, and we will show that the probability of reaching such a vertex is small. As every example $(a,b)$ has some small uniform noise  $\eta\sim U[-\delta,\delta]$ added to $b$, if the branching program is initially at vertex $v$ and then gets the example $(a,b)$, the posterior $\posterior$ is updated by restricting it to the thin slice of the spherical Gaussian where $a^T x = {[b-\delta,b+\delta]}$. We need to argue that this slicing does not significantly increase the inner product with the posterior $\significant$ corresponding to the smaller, target Gaussian.  This holds, provided the target Gaussian does not have significantly higher probability mass in the slice to which we are restricting.  This is easy to analyze in this special setting where the posteriors are spherical Gaussians, by simply analyzing the projections of the two Gaussians along a random direction $a$.  In our actual proof, to bound the rate of progress via this argument, we cannot assume that the posteriors have such a nice form.  Nevertheless, we show a concentration result that guarantees that, for any distribution with sufficiently small $\ell_2$ norm, the projections can not behave too much worse than projections of spherical Gaussians with the corresponding $\ell_2$ norms.


To sketch the argument more formally, we need to define some notation. We define $\td{f}$ as the point-wise product of the distributions $\posterior$ and $\significant$, with suitable normalization. Hence for any $x'$ on the $d$ dimensional unit sphere,
$$
\td{f}(x') = \frac{\posterior(x')\cdot \significant(x')}{\int\displaylimits_{z}\posterior(z)\cdot \significant(z)\; dz}\;.
$$
Let $I_{\delta}(b)$ be the interval $[b-\delta,b+\delta]$. For any distribution $f$ and fixed $a$, define   $G_{{f},a}(I_{\delta}(b)) = {\E_{x'\sim f}\Big[ \mathbf{1}(a^Tx'\in I_{\delta}(b))\Big]}$. Note that for a vertex $v$ with posterior distribution $\posterior$, $G_{\posterior,a}(I_{\delta}(b))$ is the probability mass on vectors $x'$ which are consistent with the example $(a,b)$, up to the noise level $\delta$. With some technical work, we can approximately relate $\ip{\posterior}{\significant}$ and $\ip{\edgeposterior}{\significant}$ for an edge $e$ labelled by $(a,b)$ as follows,
$$
\ip{\edgeposterior}{\significant} \approx \ip{\posterior}{\significant} \cdot \frac{G_{\td{f},a}(I_{\delta}(b))}{G_{\posterior,a}(I_{\delta}(b))}.
$$
Intuitively, the above relation  says that the progress that the truncated path $\cT$ makes towards some target distribution $\td{f}$ after receiving example $(a,b)$ depends on the ratio of the probability mass of $\td{f}$ which is consistent with $(a,b)$, and that of $\posterior$ which is consistent with $(a,b)$. Hence in order to bound $\E_e[\ip{\edgeposterior}{\significant}^{d/2}]$ in terms of $\ip{\posterior}{\significant}^{d/2}$, our goal will be to upper bound $$ \E_{a,b}\Big[\frac{G_{\td{f},a}(I_{\delta}(b))}{G_{\posterior,a}(I_{\delta}(b))}\Big]^{d/2}.$$
Note that as $b=a^Tx+\eta$ where $x\sim \posterior$, we can show that examples $(a,b)$ where $G_{\posterior,a}(I_{\delta}(b))$ is too small have small probability. Hence we can lower bound the denominator by making the truncated path $\cT$ stop if $G_{\posterior,a}(I_{\delta}(b))$ is too small, while still ensuring that the probability of $\cT$ stopping due to this reason is small.

It is more complicated to upper bound $\E_{a,b}[G_{\td{f},a}(I_{\delta}(b))^{d/2}]$. Note that $G_{\td{f},a}(I_{\delta}(b))$ is the probability mass of the distribution $\td{f}$ which lies in the interval $I_{\delta}(b)$ when $\td{f}$ is projected onto a random direction $a$. The linear projection of a high-dimensional distribution onto a random direction is a well-studied topic, and it is known that under mild conditions on $\td{f}$ such as bounded second moments, its projection onto a random direction is approximately Gaussian \citep{bobkov2003concentration,anttila2003central,von1997sudakov} or a mixture of Gaussians \citep{dasgupta2006concentration} with high probability. However, these results typically only give an additive $O(1/\sqrt{d})$ error guarantee for the difference between the probability mass of $\td{f}$ on any interval $I$ and that of an appropriate Gaussian on that interval (and this is tight given only second moment constraints). Note that in our case the intervals $I_{\delta}(b)$ have exponentially small width $\delta$, and we care about the multiplicative approximation error, hence these $O(1/\sqrt{d})$ additive error guarantees are not strong enough. We show that we can obtain stronger guarantees in our case by ensuring that $\norm{\td{f}}$ is small, which we guarantee by appropriate conditions on the truncated path $\cT$. With a bound on $\norm{\td{f}}$, we prove the following concentration result for projections of high-dimensional distributions onto a random direction---

\begin{restatable}{lem}{upbndsmpl}\label{lem:upper_bnd_simple}
	Let $\td{f}$ be a distribution over the $d$ dimensional sphere, with $\norm{\td{f}}\le \frac{(100/\eps)^{d}}{C_d}$ for some $\eps\le 1$. For an absolute constant $C$ and fixed $b$, $$\E_a\Big[G_{\td{f},a}(I_{\delta}(b))^{d/2}\Big]\le (C\eps^{-20}\delta)^{d/2}.$$
\end{restatable}

Finally, note that the above bound is for a fixed $b$, but if the branching program knows $x$ to a small error then it also knows the inner product $b$ for any $a$ to a small error, hence the distribution of $b$ is itself highly dependent on $a$. To get around this, we prove a version of the above lemma where $b$ is obtained by first sampling $x$ from $\posterior$ and then adding noise $\eta$ to $a^Tx$. These concentration bounds may be useful beyond this work, and it may be interesting to further develop our understanding of properties of the projection of high-dimensional distributions with small $\ell_2$ norm onto a random direction. 


\section{Notation}

Let $\cS$ be the set of all vectors on the $d$-dimensional unit sphere, and $\cU$ be the uniform distribution over $\cS$. Hence $\cU(x)=1/C_d$ for all $x\in \cS$, for some $C_d$ which depends on $d$.

Let $E_v$ denote the event that the truncated path reaches a vertex $v$. For any random variable $Z$ we denote the distribution of $Z$ by $\Prob_Z$. We denote the probability of any vertex $v$ in the branching program by $\Pr(v)$. As the edges $e$ of the branching program are indexed by real valued $(a,b)$, for any edge $e$ we denote the p.d.f. of the distribution over all edges of the branching program evaluated at the edge $e$ by $p(e)$. Let the sample at the $i$th time step be $(a_{i},b_{i})$. Recall that the distribution of $a_i$ is $N(0,I_d)$, and we will denote its p.d.f. at a vector $a$ by $p(a)$. Similarly, we denote the p.d.f. of $b$ conditioned on being at vertex $v$ and seeing example $a$ as $p(b|a,E_v)$. For any function $f$ from $\cS\rightarrow \mathbb{R}$, we denote by $\norm{f}$ the $\ell_2$ norm of $f$ with respect to the uniform distribution $\cU$ over $\cS$,
$$
\norm{f}=\Big(\E_{x \sim \cU}[f(x)^2]\Big)^{1/2}.
$$


Recall from the previous section that $G_{{f},a}(I_{\delta}(b)) = {\E_{x'\sim f}\Big[ \mathbf{1}(a^Tx'\in I_{\delta}(b))\Big]}$ for any distribution $f$, where $I_{\delta}(b)$ is the interval $[b-\delta,b+\delta]$. For notational convenience, we will subsequently denote $G_{\posterior,a}(I_{\delta}(b))$ for a vertex $v$ by $G_{v,a}(I_{\delta}(b))$.
\section{Proof of Theorem \ref{thm:main}}


In this section, we formally define the stages of the branching program and the truncated computation path $\cT$, and then provide a proof for Theorem \ref{thm:main}.
 
\subsubsection*{Stages of the Branching Program} 
 
We partition the branching program $B$ into $T$ stages $\{B_t:0\le t\le T\}$, for some $T$ which depends on the desired accuracy $\eps$. The  $t$th stage $B_t$ continues for $m_{t}$ time steps, where $m_t=\lceil \frac{c_0d}{\log d +t} \rceil$ and $c_0$ is an absolute constant to be determined later. We define the stages inductively. The first stage $B_0$ consists of all vertices up until and including the $m_0$-th layer of the branching program $B$. The $t$-th stage $B_t$ consists of $(m_t+1)$ layers beginning with and including the last layer of the previous stage $B_{t-1}$. 

\subsubsection*{Truncated Path} 

We define the truncated path $\cT$ corresponding to the branching program $B$. The truncated path $\cT$ follows the same path as $B$, except that it sometimes stops before reaching a leaf vertex. 
The conditions under which the truncated path stops before reaching a leaf vertex will be different depending on the stage $t$. Define $\eps_t=2^{-t}$. Intuitively, $\eps_t$ determines the accuracy to which $B$ could know $x$ in the $t$-th stage. In the $t$-th stage $B_t$ of $B$, the truncated path stops at a non-leaf vertex $v$ for any of the following three reasons---

 
\begin{enumerate}
	\item If $v$ is a \emph{significant} vertex, where $\norm{\posterior} > \frac{(2/{\eps_t})^d}{C_d}$.
	\item If $x$ belongs to the set of vectors $\Sig(v)$ which have non-trivial probability mass under $\posterior(x)$, defined as $\Sig(v)=\Big\{x':\posterior(x')> \frac{(4/{\eps_t})^{2d}}{C_d}\Big\}$.
	\item If the branching program is about to traverse a \emph{bad} edge. The set $\Bad(v)$ of bad edges for the vertex $v$ is defined as the set of edges $(a,b)$ for which either i) $\norm{a}\ge 2\sqrt{d}$, or
ii) $G_{v,a}(I_\delta(b))\le 2\delta/d^3$.
\end{enumerate}

If the truncated path $\cT$ does not stop at a non-leaf vertex, then it follows the same path as the computation path of the branching program $B$. Lemma \ref{lem:stage_stop} proved in Section \ref{sec:truncated_stop} shows that the probability of the truncated path stopping at a non-leaf vertex is small.

\begin{restatable}{lem}{stoppath}\label{lem:stage_stop}
	If the number of samples $m\le d^{1.25}$ and the width of the branching program $w\le 2^{d^2/4}$, then the probability of $\cT$ stopping at a non-leaf vertex is at most $1/(2d)$.
\end{restatable}

To prove Lemma \ref{lem:stage_stop}, we show that the probability of the truncated path stopping at a non-leaf vertex due to each of the three above reasons is small. Most of the effort goes into proving that the probability of stopping due to the first reason, reaching a significant vertex, is small. This is proved in Section \ref{sec:sig}. Using Lemma \ref{lem:stage_stop}, we are now ready to prove our main theorem.

\begin{restatable}{thm}{mainthm}\label{thm:main}
	Let $B$ be a branching program to find  $\td{x}:\norm{x-\td{x}}\le \eps$, where $\eps= 1/d^r$ for some $r\le O(d/\log d)$. For a small absolute constant $c$, if $B$ has length at most $cd\log r$ and width at most $2^{d^2/4}$, then the success probability of $B$ is at most $1/d$.
\end{restatable}
\begin{proof}
	We partition the branching program into $T$ stages and consider the truncated path $\cT$. We first bound the number of stages $T$ required to do the partition if $m\le cd\log r$. We claim that $T\le (r/40)\log d$. As the $t$-th stage consists of $\lceil \frac{c_0d}{\log d +t}\rceil $ steps, the number of steps in $T=(r/40)\log d$ stages can be lower bounded as follows,


	\begin{align*}
	\sum_{t=1}^{(r/40)\log d} \left\lceil \frac{c_0d}{\log d+t}\right\rceil &\ge c_0d\sum_{t=\log d}^{(r/40)\log d}\frac{1}{\log d+t}\ge \frac{c_0d}{2}\sum_{t=\log d}^{(r/40)\log d}\frac{1}{t}\\
	&\ge \frac{c_0d}{2}\log\Big(\frac{r \log d}{40 \log d}\Big) \ge  \frac{c_0d\log r}{100}.
	\end{align*}
	Hence taking $c=c_0/100$, the number of stages $T$ in $cd\log r$ steps is at most $(r/40)\log d$. Note that if $\cT$ does not stop before reaching a leaf, then it follows the same path as the branching program $B$. By Lemma \ref{lem:stage_stop}, the probability that $\cT$ stops before reaching a leaf is at most $1/(2d)$. Hence we now only need to bound the probability that a non-significant leaf $v$ of $\cT$ outputs $\td{x}$ such that $\norm{x-\td{x}}\le \eps$. However, for a non-significant leaf $v$ we know that $\norm{\posterior}\le \frac{(2/\eps_T)^d}{C_d}$. Further, the following lemma (proved in Section \ref{sec:rate}) shows that this condition implies that the probability of $v$ outputting an $\eps$ accurate answer is small.
	
	\begin{lem}\label{lem:sphere_simpe}
		Let $f$ be a distribution over the $d$ dimensional sphere $\cS$, with $\norm{f}\le \frac{(2/\eps)^{d}}{C_d}$ for some $\eps\le 0.01$. Then for any $x\in\cS$,
		$$\Pr_{x\sim f}\Big[x:\norm{x-\td{x}}\le \eps^{40}\Big]\le 2^{-d/2}.$$
	\end{lem}
	Now for $T=(r/40)\log d$, $\eps_T = 1/d^{r/40}$. Hence by Lemma~\ref{lem:sphere_simpe}, the probability that a non-significant leaf $v$ outputting an $1/d^{r}$ accurate answer is at most $2^{-d/2}$. 
	By a union bound over the probability of the truncated path stopping before a non-leaf vertex and the probability of a non-significant leaf outputting a valid answer, the probability of $B$ outputting an $1/d^{r}$ accurate answer is at most $1/d$. 
\end{proof}

We also note that the lower bound on Euclidean distance in Theorem \ref{thm:main} also readily implies a lower bound on suboptimality in terms of the expected squared prediction error: this is because for standard Gaussian inputs $a$, ground truth vector $x$ and any vector $\tilde{x}$, $\E_a[(\ip{a}{x}-\ip{a}{\tilde{x}})^2] = \norm{x-\tilde{x}}^2$.
\section{Bounding the Probability of the Truncated Path Stopping Early}\label{sec:truncated_stop}

In this section, we show that the probability the truncated path $\cT$ stop sat a non-leaf vertex is small. Lemma~\ref{lem:final} shows that probability of $\cT$ stopping because of the first reason (reaching a significant vertex) is small. Most of the remainder of the paper will be devoted to proving Lemma \ref{lem:final}.

\begin{restatable}{lem}{sig}\label{lem:final}
	If the total number of stages $T\le d^{1.25}$ and the width of the branching program $w\le 2^{d^2/4}$, then the probability that $\cT$ reaches a significant vertex in any stage is at most $2^{-d}$.
\end{restatable}

Lemma \ref{lem:stop2} and Lemma \ref{lem:stop3} show that the probability of the truncated path $\cT$ stopping due to reasons (2) and (3) respectively is small.


\begin{lem}\label{lem:stop2}
	If $v$ is not a significant vertex of $B$, then
	\begin{align*}
	\Pr[x \in \Sig(v)|E_v] \le 2^{-2d}.
	\end{align*}
\end{lem}
\begin{proof}
	Assume $v$ is in the $t$-th stage in the branching program. Since $v$ is not a significant vertex,
	\begin{align*}
	\E_{x'\sim \posterior}[\posterior(x')]=\int\displaylimits_{x'\in \cS}^{}\posterior(x')^2 dx'=C_d \E_{x' \sim \cU}[\posterior(x')^2] \le \frac{(2/{\eps_t})^{2d}}{C_d}.
	\end{align*}
	Hence by Markov's inequality,
	\begin{align*}
	\Pr_{x'\sim \posterior}\Big[\posterior(x')>\frac{(4/{\eps_t})^{2d}}{C_d}\Big]\le 2^{-2d}.
	\end{align*}
	Since conditioned on $E_v$, the distribution of $x$ is $\posterior$, we get,
	\begin{align*}
	\Pr[x \in \Sig(v)|E_v]=\Pr_{x'\sim \posterior}\Big[\posterior(x')>\frac{(4/{\eps_t})^{2d}}{C_d}\Big]\le 2^{-2d}.
	\end{align*}
\end{proof}

\begin{lem}\label{lem:stop3}
	$\Pr_{(a_{i+1},b_{i+1})}[(a_{i+1},b_{i+1})\in \Bad(v)]\le 5/d^{2.5}$.
\end{lem}
\begin{proof}
	As $a\sim N(0,I)$, by standard concentration bounds for $\chi^2$ random variables, $\Pr[\norm{a}\ge 2\sqrt{d}]\le e^{-d/10}$. Conditioned on $\norm{a}\le 2\sqrt{d}$, $|a^Tx|\le 2\sqrt{d}$. As $b$ is generated by adding noise drawn uniformly at random from $[-\delta,\delta]$ to the true inner product $a^Tx$, $p(b|a,E_v)=(2\delta)^{-1}G_{v,a}(I_\delta(b))$, where we use our notation $G_{v,a}(I_\delta(b))=\E_{x'\sim \posterior}[\mathbf{1}(a^Tx'\in I_{\delta}(b))]$. Let $u(b)$ be the p.d.f. of the uniform distribution on $b$ with support $[-2\sqrt{d},2\sqrt{d}]$. Note that,
	\begin{align*}
	\int\displaylimits_{b\in \mathbb{R}}^{}p(b|a,E_v) \mathbf{1}\Big(p(b|a,E_v)\le 4u(b)/d^{2.5}\Big)\;db \le 4\int\displaylimits_{b\in \mathbb{R}}^{}\frac{u(b)}{d^{2.5}} \;db \le \frac{4}{d^{2.5}}.
	\end{align*}
	Therefore as $u(b)=1/(4\sqrt{d})$ and $G_{v,a}(I_\delta(b))=2\delta p(b|a,E_v)$,
	\begin{align*}
	\int\displaylimits_{b\in \mathbb{R}}^{}p(b|a,E_v) \mathbf{1}\Big(G_{v,a}(I_\delta(b))\le 2\delta/d^{3}\Big)\;db  \le \frac{4}{d^{2.5}}.
	\end{align*}
	By a union bound, it follows that $\Pr_{(a_{i+1},b_{i+1})}[(a_{i+1},b_{i+1})\in \Bad(v)]\le 5/d^{2.5}.$
\end{proof}

Using these results, we can show that the probability of $\cT$ stopping at a non-leaf vertex is small---

\stoppath*
\begin{proof}
	By Lemma \ref{lem:final}, the probability that $\cT$ reaches a significant vertex and hence stops due to the first reason is at most $2^{-d}$. If $\cT$ does not reach a significant vertex, then by Lemma \ref{lem:stop2}, the probability of stopping due to the second reason at any non-significant vertex is at most $2^{-2d}$. Taking a union bound over all the $m\le d^{1.25}$ steps, the probability of stopping due to the second reason is at most $2^{-d}$. By Lemma \ref{lem:stop3}, the probability of getting a bad sample $(a,b)$ at any time step and hence stopping due to the third reason is at most $5/d^{2.5}$. Taking a union bound over the $m\le d^{1.25}$ time steps, the probability of stopping due to the third reason at any time step is at most $5/{d^{1.25}}$. Hence the overall probability of the truncated path $\cT$ stopping at a non-leaf vertex is at most $1/(2d)$. 
\end{proof}

\section{Bounding the Probability of Significant Vertices}\label{sec:sig}

In this section, we bound the probability of the truncated path $\cT$ reaching a significant vertex in the $t$-th stage, for any $t$. We begin by first finding an expression for the posterior distribution $\edgeposterior$ of $x$ conditioned on traversing an edge $e$, and then upper bound the norm of a significant vertex $s$ in the $t$-th stage $B_t$ of $B$.

\subsubsection*{Relating $\posterior$ and $\edgeposterior$, and bounding $\norm{\significant}$}

We relate $\posterior$ and $\edgeposterior$. Recall that $I_\delta(b)$ is the interval $[b-\delta,b+\delta]$. We claim that,

\begin{lem}\label{lem:edge_posterior}
For any ${\edge}$ labeled by $(a,b)$, such that $p(e)>0$,
		\[ 
		\edgeposterior(x')=
		\begin{cases} 
		\posterior(x')/c_e & \text{ if } x'\notin \Sig(v) \text{ and } a^Tx'\in I_\delta(b) \\
		0 & \text{ if } x'\in \Sig(v) \text{ or } a^Tx' \notin I_\delta(b)
		\end{cases}
		\]
	where $c_e \ge \delta/d^3$.
\end{lem}
\begin{proof}
	Let ${\edge}$ be an edge labeled by $(a,b)$, such that $p(e)>0$. Since $p(e)>0$, the vertex $v$ is not significant, as otherwise $\cT$ stops on $v$. Also, as $p(\edge)>0$, ${\edge}\notin \Bad(v)$, as otherwise $\cT$ never traverses edge ${\edge}$.
	
	If $\cT$ reaches $v$ it traverses the edge ${\edge}$ if and only if: $x\notin \Sig(v)$ (as otherwise $\cT$ stops on $v$) and the next sample received is $(a,b)$. Also, note that $b=a^Tx+\eta$, where the noise $\eta$ is uniform on $[-\delta,\delta]$. Hence the set of $x'$ which are consistent with the example $(a,b)$ are those where $a^Tx'\in I_{\delta}(b)$. Therefore for any $x'\in \cS$,
	\[ 
	\edgeposterior(x')=
	\begin{cases} 
	\posterior(x')/c_e & \text{ if } x'\notin \Sig(v) \text{ and } a^Tx'\in I_\delta(b) \\
	0 & \text{ if } x'\in \Sig(v) \text{ or } a^Tx' \notin I_\delta(b)
	\end{cases}
	\]
	where $c_e$ is a normalization factor, given by
	$$
	c_e = \int\displaylimits_{x':x' \notin \Sig(v) \wedge a^Tx\in I_\delta(b)} \posterior(x')\; dx' = \Pr_{x}[(x \notin \Sig(v) \wedge a^Tx\in I_\delta(b) | E_v )].
	$$ 
	Since $v$ is not significant, by Lemma \ref{lem:stop2}, 
	$$
	\Pr_x[x\in \Sig(v)|E_v] \le 2^{-2d}.
	$$
	Also, since $(a,b)\notin \Bad(v)$, 
	$$
	\Pr_x[a^Tx\notin I_\delta(b)|E_v] \le 1-2\delta/d^3.
	$$
	Hence by a union bound and using the fact that $\delta \ge 2^{-d/5}$,
	$$
	c_e \ge 1-(1-2\delta/d^3+2^{-2d}) \ge \delta/d^3.
	$$
\end{proof}

We now show that $\norm{\significant}$ cannot be too large. To show this, we will first show that $\norm{\edgeposterior}$ cannot be too large for any edge ${\edge}$ such that the $p(e)>0$.

\begin{lem}\label{lem:bound_norm_edge}
	For any edge ${\edge}$ in the $t$-th stage $B_t$ of the branching program such that $p({\edge})>0$,
	$\norm{\edgeposterior} \le
	\frac{(d^3/\delta) (2/\eps_t)^d}{C_d} $.
\end{lem}
\begin{proof}
	Let ${\edge}$ be an edge of the branching program from vertex $v$ to vertex $u$ such that $p(\edge)>0$. Since $p(e)>0$, the vertex $v$ is not significant (as otherwise $\cT$ stops on $v$ and $p(\edge)=0$). As $v$ is not significant,
	$$
	\norm{\posterior}\le \frac{(2/\eps_t)^d}{C_d}.
	$$
	By Lemma \ref{lem:edge_posterior},
	\[ 
	\edgeposterior(x')=
	\begin{cases} 
	\posterior(x')/c_e & \text{ if } x'\notin \Sig(v) \text{ and } a^Tx'\in I_\delta(b) \\
	0 & \text{ if } x'\in \Sig(v) \text{ or } a^Tx' \notin I_\delta(b)
	\end{cases}
	\]
	where $c_e \ge \delta/d^3$. Therefore,  	$\norm{\edgeposterior} \le
	\frac{(d^3/\delta) (2/\eps_t)^d}{C_d} $.
\end{proof}
 
We now use Lemma \ref{lem:bound_norm_edge} to bound $\norm{\significant}$. 
 
\begin{lem}\label{lem:bound_norm}
	For any significant vertex $s$ in the $t$-th stage $B_t$ of the branching program, $\norm{\significant} \le
	\frac{(d^3/\delta) (2/\eps_t)^d}{C_d} $.
\end{lem}
\begin{proof}
Let $s$ be a significant vertex in $B_t$. Let $\Gamma_{in}(s)$ be the set of edges ${{\edge}}$ going into $s$. We can write,
	$$
	\int\displaylimits_{{\edge}\in \Gamma_{in}(s)}^{}p(\edge)\;de=\Pr(s).
	$$
	By the law of total probability, for every $x'\in \cS$,
	\begin{align*}
	\significant(x')&=\int\displaylimits_{{\edge}\in \Gamma_{in}(s)}^{}\frac{p(\edge)}{\Pr(v)}\cdot \edgeposterior(x')\; de.
	\end{align*}
	By using Jensen's inequality,
	\begin{align*}
	\significant(x')^2&\le\int\displaylimits_{{\edge}\in \Gamma_{in}(s)}^{}\frac{p(\edge)}{\Pr(v)}\cdot \edgeposterior(x')^2\; de.
	\end{align*}
	Summing over all $x'\in \cS$,
	\begin{align*}
	\norm{\significant}^2&\le\int\displaylimits_{{\edge}\in \Gamma_{in}(s)}^{}\frac{p(\edge)}{\Pr(v)}\cdot \norm{\edgeposterior}^2\; de.
	\end{align*}
	By Lemma \ref{lem:bound_norm_edge}, for any edge ${\edge}\in \Gamma_{in}(s)$, $\norm{\edgeposterior}^2 \le\Big(
	\frac{(d^3/\delta) (2/\eps_t)^d}{C_d}\Big)^2$. Hence 	$\norm{\significant} \le
	\frac{(d^3/\delta) (2/\eps_t)^d}{C_d} $. 
\end{proof}
 
\subsubsection*{Similarity to a Target Distribution}

To show that the probability of $\cT$ reaching a significant vertex is small, we will argue that the posterior of $x$ on seeing a new example is not significantly similar to the target posterior distribution of a significant vertex. We use the inner product of two distributions to measure their similarity, and define it as follows. For two functions $f,g:\cS \rightarrow \R^+$, define the inner product  $$\ip{f}{g}=\E_{z \in \cS} [f(z)g(z)].$$
Note that for a significant vertex $s$ in the $t$-th stage,
\begin{align}
\ip{\significant}{\significant}= \norm{\significant}^2 > \frac{(2/{\eps_t})^{2d}}{C_d^2}.\label{eq:sig_ip}
\end{align}
We now bound the inner product of $\significant$ with all states $v_0$ in the first layer of the $t$-th stage $B_t$ of $B$.
\begin{lem}\label{lem:ini_ip}
	For all states $v_0$ with $\Pr(v_0)>0$ in the first layer of the $t$-th stage $B_t$ of $B$,
	\begin{align*}
	\ip{\initial}{\significant}\le \frac{(d^3/\delta)(\sqrt{2}/{\eps_t})^{2d}}{C_d^2}.
	\end{align*}
\end{lem}
\begin{proof}
	We claim that $\norm{\initial}\le \frac{(d^3/\delta)(1/{\eps_t})^{d}}{C_d}$ for all states $v_0$ in the first layer of $B_t$. Consider the $(t-1)$-th stage $B_{t-1}$ of the branching program $B$. The truncated path $\cT$ stops at any significant vertex, and recall that a significant vertex for the $(t-1)$th stage is defined as a vertex $s$ where $$\norm{\significant}>\frac{(2/\eps_{t-1})^{d}}{C_d}.$$ 
	Hence for all non-significant vertices $v$ in the $(t-1)$-th stage $B_{t-1}$ of $B$,
	$$
	\norm{\posterior}\le \frac{(2/\eps_{t-1})^{d}}{C_d}. 
	$$ 
	Also, by Lemma \ref{lem:bound_norm} for all significant vertices $s$ in $B_{t-1}$, 
	$$
	\norm{\significant}\le \frac{(d^3/\delta)(2/\eps_{t-1})^{d}}{C_d}. 
	$$
	Hence for all vertices $v$ in $B_{t-1}$ with $\Pr(v)>0$,
	$$
		\norm{\posterior}\le \frac{(d^3/\delta)(2/\eps_{t-1})^{d}}{C_d}. 
	$$
	Note that $\eps_{t-1}=2\eps_t$, hence $\norm{\initial}\le \frac{(d^3/\delta)(1/{\eps_t})^{d}}{C_d}$ for all states $v_0 \in L_0$ with $\Pr(v_0)>0$, as $L_0$ is also the last layer of $B_{t-1}$. Now by using Cauchy Schwartz,
	$$
	\ip{\initial}{\significant}\le \frac{(d^3/\delta)(\sqrt{2}/{\eps_t})^{2d}}{C_d^2}.
	$$
\end{proof}

Note that the inner product of $\significant$ with itself is larger than the inner product of $\significant$ with $\posterior$ for $v$ in the first layer by a factor of about $2^{\Theta(d)}$, and in the next section we will argue that this inner product cannot increase too quickly in a small number of time steps, via a suitable potential function.

\subsection*{Progress Towards Target Distribution}

In this section, we bound how much progress $\cT$ can make towards a significant vertex $s$ in the $t$-th stage $B_t$ of $B$. For notational convenience, we will reindex all the layers in the $t$-th stage $B_t$ so that the first layer in $B_t$ is labelled as $L_0$.

Let $L_i$ denote the set of all vertices in the $i$th layer of the $t$-th stage $B_t$, with $\Pr(v)>0$. Let $\Gamma_i$ denote the set of all edges from the $(i-1)$th layer to the $i$th layer of $B_t$. For $i \in \{0,\dots,m_{t}\}$ and $\beta=1/2$, let
\begin{align}
 \cZ_i = \sum_{v \in L_i}^{}\Pr(v)\cdot\ip{\posterior}{\significant}^{\beta d} .\label{eq:potential1}
 \end{align}
For $i \in \{0,\dots,m_t\}$, let
\begin{align}
\cZ_i' = \int\displaylimits_{\edge \in \Gamma_{i}}p(\edge)\cdot\ip{\edgeposterior}{\significant}^{\beta d}\; de.\label{eq:potential2}
\end{align}
Note that by Lemma \ref{lem:ini_ip},
\begin{align}
\cZ_0 \le \Big(\frac{(d^3/\delta)(\sqrt{2}/{\eps_t})^{2d}}{C_d^2}\Big)^{\beta d}.\label{eq:ini_z}
\end{align}
The goal of the next three Lemmas is to bound how much $\cZ_i$ can increase at every step. Lemma \ref{lem:bound_progress} does most of the heavy-lifting, and shows that for a fixed vertex $v$, the contribution to the potential $\cZ_i'$ from $v$'s outgoing edges is not much larger than $v$'s contribution to $\cZ_i$. Lemma \ref{lem:bound_edge_progress} uses Lemma \ref{lem:bound_progress} to show that $\cZ_i'$ is not much larger than $\cZ$. Finally, Lemma \ref{lem:convex} shows that $\cZ_{i+1}\le \cZ_i$ by a convexity argument.

\begin{lem}\label{lem:bound_progress}
	Consider the $t$-th stage $B_t$ of $B$, and let $s$ be a significant vertex in $B_t$. For every vertex $v$ of $B_t$ such that $\Pr(v)>0$,
	$$ \int\displaylimits_{\edge \in \Gamma_{out}(v)}^{}\frac{p(\edge)}{\Pr(v)}\cdot \ip{\edgeposterior}{\significant}^{\beta d} \;de\le \ip{\posterior}{\significant}^{\beta d}\cdot(C'\eps_t^{-20}d^3)^{\beta d} + \Big(\frac{d^3(1/\eps_t)^{2d}}{\delta C_d^2}\Big)^{\beta d},$$
	where $C'$ is an absolute constant.
\end{lem}

\begin{proof}
	If $v$ is a significant vertex or if $v$ is a leaf of the branching program, then $\Gamma_{out}(v)$ is the empty set and hence the claim is trivially true. Hence we will assume that $v$ is not a significant vertex or a leaf.
	
	Define $P:\cS\rightarrow \mathbb{R}^+$ as follows. For any $x'\in \cS$,
	\begin{align*}
	P(x')= \begin{cases} 
	0 & \text{ if } x'\in \Sig(v) \\
	\posterior(x') & \text{ if } x'\notin \Sig(v).
	\end{cases}	
	\end{align*}
	Note that by the definition of $\Sig(v)$, for any $x'\in \cS$,
	\begin{align}
	P(x')\le \frac{(4/\eps_t)^{2d}}{C_d}.\label{eq:bound_max}
	\end{align}
	Define $f:\cS\rightarrow \mathcal{R}^+$ as follows. For any $x'\in\cS$,
	\begin{align*}
	f(x') = P(x')\cdot \significant(x').
	\end{align*}
	By Lemma \ref{lem:bound_norm} and Eq. \eqref{eq:bound_max}, 
	\begin{align}
	\norm{f}\le \inftynorm{P}\norm{\significant}\le \frac{d^3
		(4/\eps_t)^{3d}}{\delta C_d^2}. \label{eq:bound_norm}
	\end{align}
By Lemma \ref{lem:edge_posterior}, for any edge ${{\edge}}\in \Gamma_{out}(v)$ labeled by $(a,b)$ and for any $x'\in \cS$,
\begin{align*}
\edgeposterior(x')=
\begin{cases} 
0 & \text{ if } a^Tx'\notin I_{\delta}(b) \\
P(x')\cdot c_e^{-1} & \text{ if } a^Tx'\in I_{\delta}(b)
\end{cases}	
\end{align*}
where $c_e\ge \delta/d^3$. Hence for any edge ${{\edge}}\in \Gamma_{out}(v)$ labeled by $(a,b)$ and any $x'\in \cS$ we can write,
\begin{align*}
\edgeposterior(x')\significant(x')=
\begin{cases} 
0 & \text{ if } a^Tx'\notin I_{\delta}(b) \\
f(x')\cdot c_e^{-1} & \text{ if } a^Tx'\in I_{\delta}(b).
\end{cases}	
\end{align*}
Let 
\begin{align*}
F=\int\displaylimits_{x'\in\cS}f(x')\; dx'.
\end{align*}
Recall that for uniform distribution $\cU$ on $\cS$, $\cU(x)=1/C_d$ for all $x\in\cS$. Hence we can write,
\begin{align}
\ip{\edgeposterior}{\significant}&= \E_{x' \sim \cS} [ \edgeposterior(x')\cdot \significant(x') ]\nonumber\\
&= (c_e C_d)^{-1}{\int\displaylimits_{\{x':a^Tx'\in I_{\delta}(b)\}}f(x') \;dx'}{}.\label{eq:xgivene}
\end{align}
We will now bound $\ip{\edgeposterior}{\significant}$ by considering two separate cases:\\

\noindent\emph{Case 1: } $F\le (1/\eps_t)^{2d}/C_d$.\\

In this case, we bound $\int\displaylimits_{\{x':a^Tx'\in I_{\delta}(b)\}}f(x')\;dx'  \le \int\displaylimits_{\{x'\in \cS\}} f(x')\;dx'  = F$. As $c_e^{-1}\le d^3/\delta$, using Eq. \eqref{eq:xgivene},
\begin{align*}
\ip{\edgeposterior}{\significant}\le \frac{d^3(1/\eps_t)^{2d}}{\delta C_d^2}.
\end{align*}
Note that for an edge ${{\edge}}$ labeled by $(a,b)$, $\frac{p({{\edge}})}{\Pr(v)}\le p(a)p(b|a,E_v)$, where the equality may not be true as the branching program could stop at $v$ if $e\in \Bad(v)$.
Hence $\int\displaylimits_{e\in \Gamma_{out}(v)}^{}\frac{p({{\edge}})}{\Pr(v)}\;de\le 1$ and Lemma \ref{lem:bound_progress} follows in this case.\\

\noindent\emph{Case 2: } $F> (1/\eps_t)^{2d}/C_d$.\\

We rewrite $\ip{\edgeposterior}{\significant}$ as follows,
\begin{align}
\ip{\edgeposterior}{\significant}
&=  F/C_d \cdot c_e^{-1}{\int\displaylimits_{\{x':a^Tx'\in I_{\delta}(b)\}}f(x')/F\;dx' }{}.\label{eq:case2_1}
\end{align}
For every $x'\in \cS$, we define $\td{f}(x')=f(x')/F$. Note that $\int\displaylimits_{x'}\td{f}\;dx'=1$ and hence $f'$ is a distribution over $\cS$. Also, we can bound $\norm{\td{f}}$ as follows,
\begin{align*}
\norm{\td{f}}\le \norm{f}/F \le \frac{d^3(64/\eps_t)^d}{\delta C_d}\le \frac{(100/\eps_t)^d}{ C_d},
\end{align*}
where we use the fact that $\delta\ge 2^{-d/5}$ and Eq. \eqref{eq:bound_norm} to bound $\norm{f}$. By the definitions of $P$ and $f$, 
\begin{align}
F/C_d = \E_{x' \in_R \cS} [f(x')] = \ip{P}{\significant} \le \ip{\posterior}{\significant}.\label{eq:case2_2}
\end{align}
By using Eqs. \eqref{eq:case2_1} and \eqref{eq:case2_2}, we can write,
\begin{align*}
\int\displaylimits_{\edge \in \Gamma_{out}(v)}^{}\frac{p(\edge)}{\Pr(v)} &\ip{\edgeposterior}{\significant}^{\beta d}\; de \le \int\displaylimits_{\edge \in \Gamma_{out}(v)}^{}\frac{p(\edge)}{\Pr(v)}
\Big(F/C_d \cdot c_e^{-1}{\int\displaylimits_{\{x':a^Tx'\in I_{\delta}(b)\}}\td{f}(x')\;dx' }{}\Big)^{\beta d}\;de\\
&= (F/C_d)^{\beta d}c_e^{-\beta d} \int\displaylimits_{\edge \in \Gamma_{out}(v)}^{}\frac{p(\edge)}{\Pr(v)}
\Big({\int\displaylimits_{\{x':a^Tx'\in I_{\delta}(b)\}}\td{f}(x')\;dx' }{}\Big)^{\beta d}\;de\\
&\le \ip{\posterior}{\significant}^{\beta d}c_e^{-\beta d} \int\displaylimits_{\edge \in \Gamma_{out}(v)}^{}\frac{p(\edge)}{\Pr(v)}
\Big({\int\displaylimits_{\{x':a^Tx'\in I_{\delta}(b)\}}\td{f}(x')\;dx' }{}\Big)^{\beta d}\;de.
\end{align*}
Recall that for an edge ${{\edge}}$ labeled by $(a,b)$, $\frac{p(\edge)}{\Pr(v)}\le p(a)p(b|a,E_v)$. Using our notation $\E_{x'\sim \td{f}}[\mathbf{1}(a^Tx'\in I_{\delta}(b))]=G_{\td{f},a}(I_{\delta}(b))$, we can write,
\begin{align*}
\int\displaylimits_{\edge \in \Gamma_{out}(v)}^{}\frac{p(\edge)}{\Pr(v)} \ip{\edgeposterior}{\significant}^{\beta d}\; de &\le \ip{\posterior}{\significant}^{\beta d}c_e^{-\beta d} \int\displaylimits_{a \in \mathbb{R}^d}^{}\int\displaylimits_{b\in \R} p(a)\; p(b|a,E_v)\;
G_{\td{f},a}(I_{\delta}(b))^{\beta d}\; db\; da.
\end{align*}
Let $\eta$ be a random variable uniform on $[-\delta,\delta]$, and $u(\eta)$ denote its distribution. Note that $p(b|a,E_v)$ is the distribution of $a^Tx'+\eta$, where $x'$ is sampled from $\posterior$ and noise $\eta$ is added to $a^Tx'$. Therefore,
\begin{align*}
\int\displaylimits_{\edge \in \Gamma_{out}(v)}^{}\frac{p(\edge)}{\Pr(v)} \ip{\edgeposterior}{\significant}^{\beta d}\; de &\le \ip{\posterior}{\significant}^{\beta d}c_e^{-\beta d}\cdot \\
&\int\displaylimits_{a \in \mathbb{R}^d}^{}\int\displaylimits_{x' \in \cS} \int\displaylimits _{\eta\in \R}p(a)\;\posterior(x')\;u(\eta)\;
 G_{\td{f},a}(I_{\delta}(a^Tx'+\eta))^{\beta d}\; d\eta \; dx'\; da\\
 &= \ip{\posterior}{\significant}^{\beta d} c_e^{-\beta d}\E_{\eta} \E_{x'\sim \posterior}\E_a\Big[ G_{\td{f},a}(I_{\delta}(a^Tx'+\eta))^{\beta d}\Big].
\end{align*}
We now use Lemma \ref{lem:upper_bnd} to bound the expectation.
\begin{restatable}{lem}{boundproj}\label{lem:upper_bnd}
	Let $f$ be a distribution over the $d$ dimensional sphere, with $\norm{f}\le \frac{(100/\eps)^{d}}{C_d}$ for some $\eps\le 1$. Let $x_0$ be any vector on the $d$ dimensional unit sphere, and $\eta_0$ be some constant. Then for any $\ell\le d/2$ and a universal constant $C$, $$\E_a\Big[G_{f,a}(I_{\delta}(a^Tx_0+\eta_0))^\ell\Big]\le (2(C/\eps)^{20}\delta)^\ell.$$
\end{restatable}
By Lemma \ref{lem:upper_bnd}, if $\beta\le 1/2$, then $\E_a G_{\td{f},a}(I_{\delta}(a^Tx'+\eta))^{\beta d}\le (2(C/\eps_t)^{-20}\delta))^{\beta d}$ for any $x'$ and $\eta$. Let $C'=2C^{20}$. Therefore,
\begin{align*}
\int\displaylimits_{\edge \in \Gamma_{out}(v)}^{}\frac{p(\edge)}{\Pr(v)} \ip{\edgeposterior}{\significant}^{\beta d}\; de &\le \ip{\posterior}{\significant}^{\beta d} c_e^{-\beta d} (C'\eps_t^{-20}\delta)^{\beta d}.
\end{align*}
As $c_e\ge \delta/d^3$, therefore,
\begin{align*}
 \int\displaylimits_{\edge \in \Gamma_{out}(v)}^{}\frac{p(\edge)}{\Pr(v)} \ip{\edgeposterior}{\significant}^{\beta d}\;de \le \ip{\posterior}{\significant}^{\beta d} (C'\eps_t^{-20}d^3)^{\beta d}.
\end{align*}

\end{proof}	

\begin{lem}\label{lem:bound_edge_progress}
	Consider the $t$-th stage $B_t$ of $B$, and let $s$ be a significant vertex in $B_t$. Consider the potential functions defined in Eqs. \eqref{eq:potential1} and \eqref{eq:potential2} which track progress towards $s$. Then for every $i \in \{1,\cdots,m_t\}$,
	\begin{align*}
	\mathcal{Z}_i' \le \mathcal{Z}_{i-1}\cdot(C'\eps_t^{-20}d^3)^{\beta d} + \Big(\frac{d^3(1/\eps_t)^d}{\delta C_d^2}\Big)^{\beta d}.
	\end{align*}
\end{lem}
\begin{proof}
	Using Lemma \ref{lem:bound_progress},
	\begin{align*}
	\mathcal{Z}_i' &= \int\displaylimits_{{{\edge}} \in \Gamma_{i}}^{}p(\edge)\cdot \ip{\edgeposterior}{\significant}^{\beta d}\;de = \sum_{v \in L_{i-1}}^{}\Pr(v)\cdot \int\displaylimits_{{{\edge}} \in \Gamma_{out}(v)}^{}\frac{p(\edge)}{\Pr(v)} \ip{\edgeposterior}{\significant}^{\beta d}\;de\\
	&\le \sum_{v \in L_{i-1}}^{}\Pr(v)\cdot\Big(\ip{\posterior}{\significant}^{\beta d}\cdot (C'\eps_t^{-20}d^3)^{\beta d} + \Big(\frac{d^3(1/\eps_t)^{2d}}{\delta C_d^2}\Big)^{\beta d}  \Big)\\
	&\le \mathcal{Z}_{i-1}(C'\eps_t^{-20}d^3)^{\beta d} + \sum_{v \in L_{i-1}}^{}\Pr(v) \Big(\frac{d^3(1/\eps_t)^{2d}}{\delta C_d^2}\Big)^{\beta d} \\
	&\le \mathcal{Z}_{i-1}(C'\eps_t^{-20}d^3)^{\beta d} + \Big(\frac{d^3(1/\eps_t)^{2d}}{\delta C_d^2}\Big)^{\beta d}.
	\end{align*}
\end{proof}

\begin{lem}\label{lem:convex}
		For every $i \in \{1,\cdots,m_t\}$,
		\begin{align*}
		\mathcal{Z}_i \le \mathcal{Z}_{i}'.
		\end{align*}
\end{lem}
\begin{proof}
	For any vertex $v$, let $\Gamma_{in}(v)$ be the set of edges ${{\edge}}=(a,b)$ going into $v$. We can write,
	$$
	\int\displaylimits_{{\edge}\in \Gamma_{in}(v)}^{}p(\edge)\;de=\Pr(v).
	$$
	By the law of total probability, for every $v\in L_i$ and every $x'\in \cS$,
	\begin{align*}
	\posterior(x')&=\int\displaylimits_{{\edge}\in \Gamma_{in}(v)}^{}\frac{p(\edge)}{\Pr(v)}\cdot \edgeposterior(x')\; de,\\
	\implies \ip{\posterior}{\significant}&=\int\displaylimits_{{\edge}\in \Gamma_{in}(v)}^{}\frac{p(\edge)}{\Pr(v)}\cdot \ip{\edgeposterior}{\significant} \; de.
	\end{align*}
	By using Jensen's inequality,
	\begin{align*}
	\ip{\posterior}{\significant}^{\beta d}&\le\int\displaylimits_{{\edge}\in \Gamma_{in}(v)}^{}\frac{p(\edge)}{\Pr(v)}\cdot \ip{\edgeposterior}{\significant}^{\beta d} \; de.
	\end{align*}
	Summing over all $v\in L_i$,
	\begin{align*}
	\cZ_i= \sum_{v \in L_i}^{}\Pr(v)\cdot \ip{\posterior}{\significant}^{\beta d}&\le \sum_{v \in L_i}^{}\Pr(v) \cdot \int\displaylimits_{{\edge}\in \Gamma_{in}(v)}^{}\frac{p(\edge)}{\Pr(v)}\cdot \ip{\edgeposterior}{\significant}^{\beta d} \; de,\\
	&= \int\displaylimits_{{\edge}\in \Gamma_{in}(v)}^{}{p(\edge)}\cdot \ip{\edgeposterior}{\significant}^{\beta d} \; de = \cZ_i'.
	\end{align*}
\end{proof}	

We now use the previous two results to bound the potential $\cZ_i$ for any layer $i$ in the stage $B_t$.

\begin{lem}\label{lem:bound_z}
		If the length of the $t$-th stage $m_t= \lceil \frac{c_0d}{\log d+ t}\rceil $ for sufficiently small constant $c_0$, then for every $i \in \{1,\cdots,m_t\}$,
		\begin{align*}
		\mathcal{Z}_i \le \frac{(d^3/\delta)^{\beta d} (1/\eps_t)^{2\beta d^2}2^{1.25\beta d^2}}{C_d^2}.
		\end{align*}
\end{lem}
\begin{proof}
	By Lemma \ref{lem:bound_edge_progress} and \ref{lem:convex},	
	\begin{align*}
	\cZ_i &\le \cZ_{i-1}\cdot(C'\eps_t^{-20}d^3)^{\beta d} + \Big(\frac{d^3(1/\eps_t)^{2d}}{\delta C_d^2}\Big)^{\beta d}.
	\end{align*}
	As $(C'\eps_t^{-20}d^3)^{\beta d}>1$, $\cZ_i$ is monotonically increasing and hence $\cZ_i\ge \Big(\frac{d^3(1/\eps_t)^{2d}}{\delta C_d^2}\Big)^{\beta d}$ for all $i>0$. Therefore we can write,
	\begin{align*}
	\cZ_i &\le \cZ_{i-1}(1+(C'\eps_t^{-20}d^3)^{\beta d})\le \cZ_{i-1}\cdot(2C'\eps_t^{-20}d^3)^{\beta d}. 
	\end{align*}
	By Eq. \eqref{eq:ini_z}, $\cZ_{0}\le \Big(\frac{(d^3/\delta)(\sqrt{2}/\eps_t)^{2d}}{C_d^2}\Big)^{\beta d}$. Hence for every $i \in \{1,\cdots,m_t\}$,
	\begin{align*}
	\mathcal{Z}_i \le \Big(\frac{(d^3/\delta)(\sqrt{2}/\eps_t)^{2d}}{C_d^2}\Big)^{\beta d}\cdot (2C'\eps_t^{-20}d^3)^{\beta d m}.
	\end{align*}
	Note that for sufficiently small $c_0$ and using the fact that $\eps_t=2^{-t}$, $m_t=\lceil \frac{c_0d}{\log d + t} \rceil = \lceil \frac{c_0d}{\log(d/\eps_t)}\rceil \le  \frac{d}{4\log(2C'\eps_t^{-20}d^3)}$. Therefore,
	\begin{align*}
		\mathcal{Z}_i &\le \Big(\frac{(d^3/\delta)(\sqrt{2}/\eps_t)^{2d}}{C_d^2}\Big)^{\beta d}\cdot {2}^{\beta d^2/4}
		=\frac{(d^3/\delta)^{\beta d} (1/\eps_t)^{2\beta d^2}2^{1.25\beta d^2}}{C_d^{2}}.
	\end{align*}
\end{proof}

\sig*
\begin{proof}
	Consider a significant vertex $s$ in the $t$-th stage $B_t$ of $B$. Assume that $s$ is in the $i$th layer of $B_t$. Then by Eq. \eqref{eq:sig_ip},
	\begin{align*}
	\cZ_i \ge \Pr(s)\cdot \ip{\significant}{\significant}^{\beta d}\ge \Pr(s) \cdot \frac{(2/\eps_t)^{2\beta d^2}}{C_d^2}.
	\end{align*}
	But by Lemma \ref{lem:bound_z},
	\begin{align*}
	\cZ_i \le \frac{(d^3/\delta)^{\beta d} (1/\eps_t)^{2\beta d^2}2^{1.25\beta d^2}}{C_d^2}.
	\end{align*}
	Therefore, 
	\begin{align*}
	\Pr(s) \le (d^3/\delta)^{\beta d}2^{-0.75\beta d^2}.
	\end{align*}
	By a union bound over the at most $d^2\cdot 2^{0.5\beta d^2}$ vertices in any stage of the branching program and using the fact that $\delta\ge 2^{-d/5}$, the probability that $\cT$ reaches a significant vertex in the $t$-th stage is at most $(d^5/\delta)^{\beta d}2^{- \beta d^2/4}\le2^{-\beta d^2/20}\le 2^{-2d}$. By taking a union bound over $T\le d^{1.25}$ stages, the probability that $\cT$ reaches a significant vertex in any stage is at most $2^{-d}$. 
\end{proof}
\section{Concentration Theorem for Projections of Distributions}\label{sec:rate}

In this section, we prove our concentration theorem for projections of high-dimensional distributions onto a random direction. 

\boundproj*
\begin{proof}
	All expectations over $\{x_i, i\in[\ell] \}$  in this proof will be with $x_i$ sampled from the distribution $f$. Recall from our definition, 
	\begin{gather*}
	 {}\E_a\Big[G_{f,a}(I_{\delta}(a^Tx_0+\eta_0))^\ell\Big] 
	 = {}\E_a\Big[ \Pi_{i=1}^\ell \E_{x_i}[\mathbf{1}_{I_{\delta}(a^Tx_0+\eta_0)}( a^Tx_i  )] \Big]
	 \end{gather*}
	 	 As the distribution of $a$ is a Gaussian and the distribution of the projection of a Gaussian along a fixed direction is well-understood, we will first interchange the order of $a$ and $x_i$ in the expectation. We can write,
	 \begin{gather*}
	 {}\E_a\Big[G_{f,a}(I_{\delta}(a^Tx_0+\eta_0))^\ell\Big] 
	 =   {}\E_{\{x_i,i \in [\ell]\}}\Big[ \E_a \Big[\Pi_{i=1}^\ell \mathbf{1}_{I_{\delta}(a^Tx_0+\eta_0)}( a^Tx_i  ) \Big] \Big].
	 \end{gather*}

		We will write the vectors $\{x_i:0\le i\le {\ell}\}$ in terms of a suitable orthogonal basis which will facilitate the analysis of the projection of $a$ onto $x_i$. For all $0\le i\le {\ell}$, let $x_i = u_i + v_i$, where $u_i$ lies in the span of $\{x_j:j< i\}$ and $v_i$ is orthogonal to the span of $\{x_j:j< i\}$. 
	 Note that because $a\sim N(0,I)$, the components of $a$ along any orthogonal basis of $\mathbb{R}^d$ are independent $N(0,1)$ random variables. Hence the components $\{v_{a,0},\dots,v_{a,{\ell}}\}$ of $a$ along the orthogonal directions $\{v_0,\dots,v_{\ell}\}$ are independent $N(0,1)$ random variables. 
	Note that $\vdot{a}{x_i}$ is independent of $\{v_{a,j}: j > i\}$, as $x_i$ is orthogonal to $\{v_j:j>i\}$. Let $a^Tx_0+\eta_0=b_0$. Using this independence, we can rewrite the expectation as follows,
	\begin{gather*}
	{}\E_{\{x_i,i \in [\ell]\}}\Big[ \E_a \Big[\Pi_{i=1}^\ell \mathbf{1}_{I_{\delta}(b_0)}( a^Tx_i  )\Big] \Big]\\
	={}\E_{\{x_i,i \in [\ell]\}}\E_{v_{a,0}}\Big[ \E_{v_{a,1}} \Big[ \mathbf{1}_{I_{\delta}(b_0)}( a^Tx_1  )
	\E_{v_{a,2}} \Big[ \mathbf{1}_{I_{\delta}(b_0)}( a^Tx_2  )
	\dots 
	\E_{v_{a,\ell}} \Big[ \mathbf{1}_{I_{\delta}(b_0)}( a^Tx_{\ell}  ) \Big] \Big] \dots\Big].
	\end{gather*}
	
		We will now upper bound $\E_{v_{a,i}} \Big[ \mathbf{1}_{I_{\delta}(b)}( a^Tx_{i}) \Big]$ for any value of $\{v_{a,j}: j <i\}$. Note that for a fixed value of the components $v_i$ and $u_i$ of $x_i$, $\vdot{v_{a,i}}{v_i}$ is a Gaussian with mean 0 and standard deviation $\norm{v_i}$ and is independent of $a^Tu_i$ as $u_i$ and $v_i$ are orthogonal. Also, note that for any value of $\vdot{a}{u_i}$  the probability that $ a^Tx_i $ lies within the interval $I_{\delta}(b)=[b-\delta, b+\delta]$ equals the probability mass of the distribution of $\vdot{v_{a,i}}{v_i}$ in the interval $I_{ \delta}(b-\vdot{a}{u_i})=[b-\vdot{a}{u_i}-\delta,b-\vdot{a}{u_i}+\delta]$. As the distribution of $\vdot{v_{a,i}}{v_i}$ is Gaussian with mean 0, the probability mass in any interval $I_{ \delta}(b-\vdot{a}{u_i})$ is upper bounded by the probability mass in the interval $I_{\delta}(0)=[-\delta,\delta]$ centered at 0. Further, because the probability mass of a Gaussian with standard deviation $\sigma$ in the interval $I_{\delta}(0)=[-\delta,\delta]$ is at most $\min\{\delta/\sigma,1\}$ and the distribution of $\vdot{v_{a,i}}{v_i}$ is a Gaussian with standard deviation $\norm{v_i}$, the probability mass in the interval $I_{\delta}(0)$ is at most $\min\{\delta/\norm{v_i},1\}\le\delta\min\{1/\norm{v_i},1/\delta\}$. Hence we can simplify the expectation as,
	\begin{gather*}
	{}\E_{\{x_i,i \in [\ell]\}}\Big[ \E_a \Big[\Pi_{i=1}^\ell \mathbf{1}_{I_{\delta}(b_0)}( a^Tx_{i})\Big] \Big]
	\le \delta^\ell{}\E_{\{x_i,i \in [\ell]\}}\Big[ \frac{1}{\Pi_{i=1}^\ell \max\{\norm{v_i},\delta\}}\Big].
	\end{gather*}
Therefore our goal now is to lower bound the component $v_i$ of $x_i$ which is orthogonal to the vectors $\{x_j, j <i \}$. This is where we will use our upper bound on $\norm{\tilde{f}}$. Intuitively, if $\tilde{f}$ is not too large then $\tilde{f}$ is not a highly concentrated distribution and hence random vectors drawn from $\tilde{f}$ will not be too close to each other. We formalize this in Lemma \ref{lem:sphere} which shows that if $\norm{\tilde{f}}$ is not too large, then the probability of $\norm{v_i}$ being small is also small.
\begin{restatable}{lem}{sphere}\label{lem:sphere}
			Let ${f}$ be a distribution over the $d$ dimensional sphere, with $\norm{{f}}\le \frac{(100/\eps)^{d}}{C_d}$ for some $\eps\le 1$. Let $\{x_0,\dots,x_{i-1}\}$ be an arbitrary set of $i$ vectors for some $i\le d/2$, and $x_i$ be a vector sampled from ${f}$. Let $v_i$ be the component of $x_i$ orthogonal to $\{x_0,\dots,x_{i-1}\}$. Then for a sufficiently large universal constant $C$,
			$$\Pr[\norm{v_i}\le (\eps/C)^{20}]\le (\eps/2)^{d/2}.$$
		\end{restatable}
The result now follows with some computation. Let $\{X_i,i\in [\ell]\}$ be independent random variables each of which take the value $\delta$ with probability $(\eps/2)^{\ell}$ and $(\eps/C)^{20}$ otherwise. Note that $\E[1/X_i]\le(1/\delta)(\eps/2)^{\ell}+(C/\eps)^{20}\le (2(C/\eps)^{20})$ for $\delta\ge2^{-d/5}$. Hence by Lemma \ref{lem:sphere},
\begin{gather*}
{}\E_{\{x_i,i \in [\ell]\}}\Big[ \frac{1}{\Pi_{i=1}^\ell \max\{\norm{v_i},\delta\}}\Big] \le \E_{\{X_i,i \in [\ell]\}}\Big[ \frac{1}{\Pi_{i=1}^\ell X_i}\Big]
=\Pi_{i=1}^\ell\E\Big[ \frac{1}{ X_i}\Big]\le (2(C/\eps)^{20})^\ell.
\end{gather*}
Hence,
$${}\E_{\{x_i,i \in [\ell]\}}\Big[ \E_a \Big[\Pi_{i=1}^\ell \mathbf{1}_{I_{\delta}(b_0)}( a^Tx_{i})\Big] \Big]\le (2(C/\eps)^{20}\delta)^\ell.$$	
\end{proof}	

We remark that the bound in Lemma \ref{lem:upper_bnd_simple} is tight up to the constant $C$ and the constant 20 in the exponent of $\eps$. This follows from the case where $\tilde{f}$ is the uniform distribution on all points $x$ on the unit sphere which are distance at most $\eps$ from a fixed point $x_0$.
	
	We now prove Lemma \ref{lem:sphere}. Lemma \ref{lem:sphere_simpe}, which is used in the proof of Theorem \ref{thm:main}, is a corollary of Lemma \ref{lem:sphere} and follows by appropriately rescaling the constants in the statement of Lemma \ref{lem:sphere}.
	
	\sphere*
	\begin{proof}
		We will first upper bound the probability that $\norm{v_i}\le (\eps/C)^{20}$ for any $i\le d/2$ if $x_i$ is drawn uniformly at random from the unit sphere. As the Gaussian distribution is spherically symmetric, we can assume without loss of generality that $\{x_j: j<i\}$ span the first $i$ basis directions. Sampling $x_i$ uniformly at random from the unit sphere is equivalent to first sampling a vector $\tilde{x_i}$ whose each coordinate is sampled independently from a $N(0,1/{d})$ distribution, and then setting $x_i=\tilde{x_i}/\norm{\tilde{x}_i}$. Using this formulation, let the $j$th coordinate of $\tilde{x}_i$ be $z_j$, where $z_j \sim N(0,1/{d})$. We will first show that with probability $1-\eps^{d}$, $\norm{\tilde{x}_i}\le 2\sqrt{1/\eps}$. This follows from the following tail bound for $\chi^2$ random variables. Note that $r^2=d\sum_{j=1}^{d}z_j^2$ is the sum of $d$ standard Gaussian random variable and hence is a $\chi^2$ random variables with $d$ degrees of freedom. We use the following concentration inequality for a $\chi^2$ random variable $r$ with $d$ degrees of freedom (Lemma 1 in \cite{laurent2000adaptive}), 
		$$ \Pr[r^2-d\ge 2\sqrt{dt}+2t] \le e^{-t},\; \forall \; t>0.$$

		Choosing $t=d\log(1/\eps)$, $$\Pr[r^2-d\ge 2d\sqrt{\log(1/\eps)}+2d\log(1/\eps)] \le \eps^{d}.$$ Using the fact that $4\max(\log(1/\eps),\sqrt{\log(1/\eps)})\le 3/\eps$ for $0\le\eps\le 1$, $\Pr[r^2-d\ge 3d/\eps] \le \eps^{d}.$ Hence $\norm{\tilde{x}_i}\le \sqrt{1+3/\eps}\le \sqrt{4/\eps}$ with failure probability $\eps^{d}$. 
		
		We now claim that ${\Pr[\sum_{j=i+1}^{d}z_i^2 \le \eps^2/64 ]} \le \eps^{d/4}$. Note that if $\sum_{j=i+1}^{d}z_i^2 \le \eps^2/64$, then by Markov's inequality $z_j^2 \le \eps^2/(32(d-i))\le \eps^2/(16d)$ for at least half of the random variables $\{z_j:j\in [i+1,d]\}$. As $z_j$ is a $N(0,1/{d})$ random variable, hence $\Pr[|z_j| \le \eps/(4\sqrt{d})] \le \eps/4$. Hence the probability that $z_j^2 \le \eps^2/(4d)$ for at least half of the random variables $\{z_j:j\in [i+1,d]\}$ is at most the probability that at least half of a set of $(d-i)$ independent coins land heads when all of them are flipped, given that each of them has probability $\eps/4$ of landing head. Let $X_j$ be the indicator random variable denoting the event $|z_j|\le \eps/(4\sqrt{d})$, note that $\Pr[X_j=1]\le \eps/4$. Let $X=\sum_{j=i+1}^{d}X_j$. Then by standard Chernoff bounds,
\begin{align*}
\Pr[X\ge (d-i)/2]\le \exp\{-(d-i)D(1/2\parallel \eps/4)\}
\end{align*}		
where $D(1/2\parallel \eps/4)$ is the relative entropy of $1/2$ with respect to $\eps/4$. Note that $D(1/2\parallel \eps/4)\ge \frac{-\ln (\eps)}{2}$ and $(d-i)\ge d/2$. 	Therefore, 
\begin{align*}
\Pr[X\ge (d-i)/2]\le \eps^{d/4}.
\end{align*}
Hence for any $i \le d/2$, $\sum_{j=i+1}^{d}z_j^2 \ge \eps^2/64$ with failure probability $\eps^{d/4}$. Note that $\norm{v_i}^2=\sum_{j=i+1}^{d}z_j^2/\norm{\td{x}_i}^2$, and we have shown that $\norm{\td{x}_i}^2\le 4/\eps $ with failure probability $\eps^{d}$. Therefore, by a union bound, $\Pr[\norm{v_i}\le \eps^{1.5}/16]\le \eps^{d/4}+\eps^{d}\le 2\eps^{d/4}$. 
For any $C\ge 200$ and $\eps\in (0,1]$, this implies that  $\Pr[\norm{v_i}\le (\eps/C)^{20}]\le (\eps/100)^{3d}$ when $x_i$ is drawn uniformly from the unit sphere. We next show that this implies that $\Pr[\norm{v_i}\le (\eps/C)^{20}]$ is also small when $x_i$ is drawn uniformly from some distribution $f$ such that $\norm{f}$ is small. 

Let $\mathcal{X}$ be the set of all $x_i$ on the unit sphere such that $\norm{v_i}\le (\eps/C)^{20}$. We have shown that for $x_i$ drawn uniformly from the unit sphere, $\Pr[x_i \in \mathcal{X}] \le (\eps/100)^{3d}$. Note that if $\mathcal{X}$ has probability $p$ under the distribution $f$, then $\norm{f}\ge \frac{p(100/\eps)^{3d/2}}{C_d}$. This is because the uniform distribution on $\mathcal{X}$ has norm $\frac{(100/\eps)^{3d/2}}{C_d}$. Therefore if $\norm{f}\le \frac{(100/\eps)^{d}}{C_d}$, then $p \le (\eps/2)^{d/2}$.


	\end{proof}



\section*{Acknowledgements}

We thank Michael Kim and the anonymous reviewers for comments and feedback. 
	
\bibliographystyle{plainnat}
\bibliography{references.bib}	

\end{document}